\newcommand{\longversion}[1]{#1}
\newcommand{\shortversion}[1]{}

\shortversion{
\relax
\documentclass[letterpaper]{article} %
\usepackage{aaai22}  %
\usepackage{times}  %
\usepackage{helvet}  %
\usepackage{courier}  %
\usepackage[hyphens]{url}  %
\usepackage{graphicx} %
\urlstyle{rm} %
\usepackage{natbib}  %
\usepackage{caption} %
\DeclareCaptionStyle{ruled}{labelfont=normalfont,labelsep=colon,strut=off} %
\frenchspacing  %
\setlength{\pdfpagewidth}{8.5in}  %
\setlength{\pdfpageheight}{11in}  %
}

\longversion{
\documentclass[10pt,letter]{article}
\usepackage[hscale=0.7,scale=0.75]{geometry}
\usepackage[USenglish]{babel}
\usepackage[round]{natbib}

\usepackage[hyphens,spaces]{url}
\usepackage{hyperref,xcolor}

\usepackage{microtype}

\newcommand{\citex}[1]{\citet{#1}}
\newcommand{\shortcite}[1]{\citep{#1}}
}

\usepackage{algorithm}
\usepackage{algorithmic}

\usepackage{amsmath}
\usepackage{amsthm}
\usepackage{amssymb}

\shortversion{\usepackage{nicefrac}}
\longversion{\newcommand{\nicefrac}[2]{\frac{#1}{#2}}}

\usepackage{booktabs}

\longversion{
  \usepackage{authblk}
}

{%
  \addtocounter{theorem}{-1}
  \endgroup
}%

{%
  \addtocounter{theorem}{-1}
  \endgroup
}%

{%
  \addtocounter{theorem}{-1}
  \endgroup
}%

\usepackage{subcaption}
\usepackage{tikz}
\usepackage{tikz-qtree}
\usetikzlibrary{trees,calc,arrows.meta,positioning,decorations.pathreplacing,bending}
\usepackage{pgfplots}
\pgfplotsset{width=5.6cm,compat=newest}
\usepackage{textcomp}
\usepackage{pifont}
\usepackage{dblfloatfix}
\usepgfplotslibrary{fillbetween}
\usepackage{siunitx} 
\usepackage{pgf}

\usepackage{todonotes}
\usepackage{xcolor}
\definecolor{lightblue}{HTML}{55AAFF}
\definecolor{lightpurple}{HTML}{DD77BB}
\definecolor{lightgray}{HTML}{DDDDDD}
\definecolor{strongyellow}{HTML}{FFE219}

\def\restrict#1{\raise-.5ex\hbox{\ensuremath|}_{#1}}

\usepackage{newfloat}
\usepackage{listings}
\floatstyle{ruled}
\newfloat{listing}{tb}{lst}{}
\floatname{listing}{Listing}
\title{Rushing and Strolling among Answer Sets -- Navigation Made Easy}
\shortversion{
\author{
    Johannes Klaus Fichte\textsuperscript{\rm 1}, 
    Sarah Alice Gaggl\textsuperscript{\rm 2}, 
    Dominik Rusovac\textsuperscript{\rm 2}
}
\affiliations{
    \textsuperscript{\rm 1} Research Unit Database and Artificial Intelligence, TU Wien, Austria\\
    \textsuperscript{\rm 2} Logic Programming and Argumentation Group, TU Dresden,  Germany\\
    \rm johannes.fichte@tuwien.ac.at, firstname.lastname@tu-dresden.de

}
}

\newtheorem{definition}{Definition}
\newtheorem{theorem}{Theorem}

\newtheorem{lemma}{Lemma}
\newtheorem{example}{Example}
\newtheorem{corollary}{Corollary}

\newcommand{\lp}{\Pi}                                         %
\newcommand{\A}[1][\lp]{\mathcal{A}(#1)}                      %
\newcommand{\AS}[1][\lp]{\mathcal{AS}(#1)}                    %
\newcommand{\BC}[1][\lp]{\mathcal{BC}(#1)}                    %
\newcommand{\CC}[1][\lp]{\mathcal{CC}(#1)}                    %
\newcommand{\F}[1][\lp]{\mathcal{F}(#1)}                      %
\newcommand{\FI}[1][\lp]{\mathcal{F}^+(#1)}                   %
\newcommand{\FE}[1][\lp]{\mathcal{F}^-(#1)}                   %
\newcommand{\SA}[2][\lp]{S^{\models_{#1} #2}}                 %
\newcommand{\DF}[1][\lp]{\mathcal{DF}(#1)}                    %

\newcommand{\route}[1][]{\langle #1 \rangle}                  %
\newcommand{\De}[2][\lp]{\Delta^{#1}_{#2}}                    %
\newcommand{\RE}[2][f]{\mathcal{R}(#2, #1)}                   %

\newcommand{\vgo}[1][]{\nu_{\mathit{go}}#1}                   %
\newcommand{\pace}[2]{\mathcal{P}_{#1}(#2)}                   %

\newcommand{\abs}{\#\mathcal{AS}}
\newcommand{\absw}{\omega_{\abs}}
\newcommand{\fc}{\#\mathcal{F}}
\newcommand{\fcw}{\omega_{\fc}}
\newcommand{\smc}{\#\mathcal{S}}
\newcommand{\smcw}{\omega_{\smc}}
\newcommand{\simnot}{\mathord{\sim}}

\newcommand{\vsgo}[1]{\nu^{#1}_{\mathit{sgo}}}
\newcommand{\vexpl}[1]{\nu^{#1}_{\mathit{expl}}}

\usepackage{mathtools}
\usepackage{xspace}

\newcommand{\complexityClassFont}[1]{\ensuremath{\mathrm{#1}}}
\newcommand{\NP}{\text{\complexityClassFont{NP}}\xspace}
\newcommand{\co}{\complexityClassFont{co}}
\newcommand{\Ptime}{\complexityClassFont{P}\xspace}
\newcommand{\PH}{\complexityClassFont{PH}\xspace}
\newcommand{\numberP}{\complexityClassFont{\#\Ptime}\xspace}

\shortversion{
  \newcommand{\citex}[1]{\citeauthor{#1}~\shortcite{#1}}
}

\newcommand{\predname}[1]{\mathit{#1}\xspace}

\begin{document}

\longversion{ \title{Rushing and Strolling among Answer Sets --
    \\Navigation Made Easy\thanks{This is the authors' self-archived
      copy, including proofs, of a paper that has been accepted for
      publication at AAAI-22.}}
\author[1]{Johannes Klaus Fichte}
\author[2]{Sarah Alice Gaggl}
\author[2]{Dominik Rusovac}
\affil[1]{\small TU Wien, Austria, \texttt{johannes.fichte@tuwien.ac.at}}
\affil[2]{\small TU Dresden, Germany, \texttt{firstname.lastname@tu-dresden.de}}
\date{}
}

\maketitle

\shortversion{
\begin{abstract}
  Answer set programming (ASP) is a popular declarative programming paradigm
  with a wide range of applications in artificial intelligence.
  Oftentimes, when modeling an AI problem with ASP, and in particular when we are
  interested beyond simple search for optimal solutions, an actual solution,
  differences between solutions, or number of solutions of the ASP program
  matter. 
  For example, when a user aims to identify a specific answer set according to
  her needs, or requires the total number of diverging solutions to comprehend
  probabilistic applications such as reasoning in medical domains.  
  Then, there are only certain problem specific and handcrafted
  encoding techniques available to navigate the solution space of ASP programs,
  which is oftentimes not enough. In this paper, we propose a formal and general
  framework for interactive navigation towards desired subsets of answer sets
  analogous to faceted browsing. Our approach enables the user to explore the
  solution space by consciously zooming in or out of sub-spaces of solutions at
  a certain configurable pace. We illustrate that weighted faceted navigation is
  computationally hard. Finally, we provide an implementation of our approach
  that demonstrates the feasibility of our framework for incomprehensible
  solution spaces.
\end{abstract}
}

\longversion{
\begin{abstract}
  \noindent Answer set programming (ASP) is a popular declarative programming paradigm
  with a wide range of applications in artificial intelligence.
  Oftentimes, when modeling an AI problem with ASP, and in particular when we are
  interested beyond simple search for optimal solutions, an actual solution,
  differences between solutions, or number of solutions of the ASP program
  matter. 
  For example, when a user aims to identify a specific answer set according to
  her needs, or requires the total number of diverging solutions to comprehend
  probabilistic applications such as reasoning in medical domains.  
  Then, there are only certain problem specific and handcrafted
  encoding techniques available to navigate the solution space of ASP programs,
  which is oftentimes not enough. In this paper, we propose a formal and general
  framework for interactive navigation towards desired subsets of answer sets
  analogous to faceted browsing. Our approach enables the user to explore the
  solution space by consciously zooming in or out of sub-spaces of solutions at
  a certain configurable pace. We illustrate that weighted faceted navigation is
  computationally hard. Finally, we provide an implementation of our approach
  that demonstrates the feasibility of our framework for incomprehensible
  solution spaces.
\end{abstract}
}

\section{Introduction}
Answer set programming (ASP) is a declarative programming paradigm, which has
its roots in logic programming and nonmonotonic reasoning. It is widely used for
knowledge representation and problem
solving~\longversion{\cite{brewka2011answer,eiter2009answer,gebser2012answer}}\shortversion{\cite{brewka2011answer}}.
In ASP, a problem is encoded as a set of rules (logic program) and is evaluated
under stable model semantics~\cite{GelfondLifschitz88,GelfondLifschitz91},
using solvers such as
\verb"clingo"~\cite{gebser2011advances,DBLP:journals/corr/GebserKKS14},
\verb"WASP"~\cite{AlvianoDodaroLeone15}, or \verb"DLV"~\cite{AlvianoEtAl17}. Then,
answer sets represent solutions to the modeled problem.

Oftentimes when modeling with ASP, the number of solutions of the
resulting program can be quite high.
This is not necessarily a problem when searching for a few solutions,~e.g.,
optimal solutions~\cite{GebserKaminskiSchaub11,AlvianoDodaro16a} or when
incorporating
preferences~\longversion{\cite{Brewka04,BrewkaEtAl15,BrewkaEtAl15b,AlvianoRomeroSchaub18}}\shortversion{\cite{Brewka04,BrewkaEtAl15,AlvianoRomeroSchaub18}}.
However, there are many situations where reasoning goes beyond simple search for
one answer set, for example, planning when certain routes are gradually
forbidden~\cite{SonEtAl16}, finding diverging
solutions~\cite{everardo2017towards,EverardoEtAl19}, reasoning in probabilistic
applications~\cite{LeeTalsaniaWang17}, or debugging answer
sets~\longversion{\cite{OetschPT18, DodaroGRRS19, VosKOPT12, Shchekotykhin15,
GebserEtAl08}}\shortversion{\cite{OetschPT18,GebserEtAl08}}.

Now, if the user is interested in more than a few solutions to
gradually identify specific answer sets, tremendous solution spaces
can easily become infeasible to comprehend.
In fact, it might not even be possible to compute all solutions in
reasonable time.
Examples where we easily see large solution spaces are configuration
problems~\cite{soininen1999developing,soininen2001representing,tiihonen2003practical}, such as for instance
PC configuration, and
planning
problems~\cite{dimopoulos1997encoding,lifschitz1999action,nogueira2001prolog}.
Let us consider a simple example to illustrate the use of navigation in ASP.
\newcommand{\pnot}{\simnot}
\begin{example}
  Consider an online shopping situation where we have a knowledge base
  on clothes and some rules which specify which combinations would suit well or not.
  \begin{align*}
\big\{ \{\predname{outfit}(X,Y): &\predname{clothes}(X,Y) \}; \\
 &\leftarrow  \predname{outfit}(X,Y1), \\ 
                                     & \quad\;\;\predname{outfit}(X,Y2), Y1 \neq Y2; & \\
\predname{occasion}(\text{vancouver}) &\leftarrow \predname{outfit}(\text{jacket},\ldots); &\\
\predname{occasion}(\text{conference}) &\leftarrow \predname{outfit}(\text{suit},Y), Y \neq \text{\text{yellow}}; &\\
   \predname{occasion}(\text{wistler}) &\leftarrow  \predname{outfit}(\text{boots},\ldots)%
   \ldots   \big\}
  \end{align*}
  Together with  input facts from a clothes database like
  $\predname{clothes}(\text{jacket,blue})$;
  $\predname{clothes}(\text{shirt,red}); \dots $ one easily obtains more than a
  million answer sets.  Since Canada opened immigration for vaccinated persons,
  we actually might be able to travel to Vancouver. Say we  
  zoom in on outfits including shorts, which leads to a rather small, but still
  incomprehensible sub-space of solutions. Imagine that most of the remaining
  outfits include chucks and a jacket. Say we want to inspect the most different
  outfits still remaining, then we aim to choose potential parts of our outfit
  that provide us with most diverse solutions. Now, we are almost good to go,
  seeking to find some final additions to our outfit quickly.

\end{example}
\noindent
Our example illustrates that different solutions in ASP programs can
easily be hard to comprehend.  Problem specific, handcrafted encoding
techniques to navigate the solution space can be quite tedious.

Instead, we propose a formal and general framework for interactive
navigation towards desired subsets of answer sets analogous to faceted
browsing in the field of information retrieval~\cite{Tunkelang09}. Our
approach enables solution space exploration by consciously zooming in
or out of sub-spaces of solutions at a certain configurable pace.
To this end we 
introduce absolute and relative weights to quantify the size of the search space
when reasoning under assumptions (facets). We formalize several kinds of search
space navigation as goal-oriented and explore modes, and systematically compare
the introduced weights regarding their usability for operations under natural
properties splitting, reliability, preserving maximal sub-spaces (min-inline),
and preserving minimal sub-spaces (max-inline). In addition, we illustrate the
computational complexity for computing the weights.  Finally, we provide an
implementation on top of the solver \verb"clingo" demonstrating the feasibility
of our framework for incomprehensible solution spaces.

\paragraph{Related Work.}
\citex{10.1007/978-3-319-99906-7_14} proposed a framework in which
solutions are systematically pruned with respect to facets (partial
solutions).
While this allows one to move within the answer set space, the
user has absolutely no information on how big the effect of activating
a facet is in advance, 
similar to assumptions in 
propositional satisfiability~\cite{EenSorensson04a}.
We go far beyond and %
characterize the \emph{weight} of a facet. 
This is useful
to comprehend the effect of navigation steps on the size of the
solution space. %
Additionally, this allows for zooming into or out
of the solution space at a configurable pace.
Debugging in answer sets has widely been investigated~\cite{OetschPT18,
   DodaroGRRS19, VosKOPT12, Shchekotykhin15,
  GebserEtAl08}.
However, we do not aim to correct ASP encodings. All answer sets which are reachable within the navigation are ``original'' answer sets,
thus the adaptions we make during the navigation to the program, do not change the set of answer sets of the initial program.
Justifications, which describe the support for the truth value of each
atom, have been studied as a tool for reasoning and
debugging~\cite{El-KhatibPontelliSon05}.
Probabilistic reasoning frameworks for logic programs were developed
such as $\text{LP}^{\text{MLN}}$~\cite{LeeTalsaniaWang17}, which
define notions of probabilities in terms of relative occurrences of
stable models and their weights. Computing these probabilities (unless
restricted to decision versions in terms of being different from zero)
relates to counting probabilities under assumptions. Considering
relative occurrences of stable models of weight one relates to search
space exploration.
However, probabilistic frameworks primarily address modeling
conflicting information %
and reason about them. We assume large solution spaces and aim for
navigating dynamically in the solution space.

\section{Background}
First, we recall basic notions of  ASP, for further
details on ASP we refer to standard texts~\cite{CalimeriFGIKKLM20,gebser2012answer}. 
Then, we introduce fundamental notions of faceted navigation and computational
complexity, respectively.
\paragraph{Answer Set Programming.}
By $\A$
we denote the set of (non-ground)
\emph{atoms} of a program $\lp$. A literal is an atom $\alpha \in
\A$ or its \emph{default negation}, which refers to the absence of
information, denoted by $\simnot\alpha$.
 An atom $\alpha$ is a
predicate $p(t_0, \dots,t_n)$ of arity $n \geq 0$
where each $t_i$ for $0 \leq i \leq n$ is a \emph{term}, i.e., either a 
variable or a constant. We say an atom $\alpha \in \A$ is \emph{ground} if and
only if $\alpha$ is variable-free. By $\mathit{Grd}(\A)$ we denote ground atoms.
A (disjunctive) logic program $\lp$
is a finite set of rules $r$ of the form
 $$\alpha_0\,|\, \ldots \,|\, \alpha_k \leftarrow \alpha_{k+1}, \dots, \alpha_m,  \simnot \alpha_{m+1}, \dots, \simnot
  \alpha_n$$
where $0 \leq k \leq m \leq n$ and each $\alpha_i \in \A$ for $0 \leq i \leq n$.
For a rule $r$ we denote the head by $H(r) \coloneqq \{\alpha_0, \dots, \alpha_k
\}$, the body $B(r)$ consists of  the positive body $B^+(r) \coloneqq \{\alpha_{k+1},
\dots, \alpha_m\}$, and the negative body $B^-(r) \coloneqq \{ \alpha_{m+1}, \dots, 
\alpha_n\}$.
If $B(r) = \emptyset$, we omit $\leftarrow$.
 A rule $r$ where $H(r) = \emptyset$ is called \emph{integrity constraint} and
 avoids that $B(r)$ is evaluated positively. 
By $\mathit{grd}(r)$ we denote the set of ground instances of some rule~$r$,
obtained by replacing all variables in $r$ by ground terms. Accordingly,
$\mathit{Grd}(\lp) \coloneqq \bigcup_{r \in \lp}\mathit{grd}(r)$ denotes the ground
instantiation of $\lp$.  Without any explicit contrary indication, throughout
this paper, we use the term (logic) program to refer to grounded disjunctive
programs where $\A = Grd(\A)$.  
An interpretation $X \subseteq \A$ satisfies a rule
$r \in \lp$ if and only if $H(r) \cap X \neq \emptyset$ whenever $B^+(r)
\subseteq X$ and $B^-(r) \cap X = \emptyset$. $X$ satisfies $\lp$, if
$X$ satisfies each rule $r \in \lp$. An interpretation $X$ is a \emph{stable model} (also called
\emph{answer set}) of $\lp$ if and only if $X$ is a subset-minimal model
satisfying the Gelfond-Lifschitz reduct of $\lp$ with respect to  $X$, defined
as $\lp_{X} \coloneqq \{H(r) \leftarrow B^+(r) \mid X \cap B^-(r) = \emptyset, r
\in \lp\}$.
By $\AS$ we denote the answer sets of~$\lp$. For computing facets, we rely on
two notions of consequences of a program, namely, \emph{brave} consequences $\BC
\coloneqq \bigcup \AS$ and \emph{cautious} consequences $\CC \coloneqq \bigcap \AS$.

\paragraph{Faceted Navigation.}
\emph{Faceted
answer set navigation} is characterized as a sequence of navigation steps
restricting the solution space with respect to partial solutions.  Those partial
solutions, called \emph{facets}, correspond to ground atoms of a program $\lp$
that are not contained in each solution. We denote the \emph{facets} of~$\Pi$ by $\F \coloneqq \FI \cup \FE$ 
where $\FI \coloneqq \BC \setminus \CC$ denotes \emph{inclusive facets} and
$\FE \coloneqq \{\overline{\alpha} \mid \alpha \in \FI\}$ denotes \emph{exclusive
facets} of $\lp$. We say an interpretation $X \subseteq \A$ satisfies an
inclusive facet $f \in \FI$, if $f \in X$, which we denote by $X \models f$, and
it satisfies an exclusive facet $f \in \FE$, if $f \not \in X$.

A navigation step is a transition from one program to another, obtained by
adding some integrity constraint that enforces the atom refered to by some
inclusive or exclusive facet to be present or absent, respectively, throughout
answer
sets. By $ic(f)$ we denote the function that translates a facet $f \in \{\alpha,
\overline{\alpha}\} \subseteq \F$ into a singleton program that contains its
corresponding integrity constraint:
    \[
            \mathit{ic}(f) \coloneqq \begin{cases} \{\leftarrow \simnot
              \alpha\}, & \text{if } f = \alpha;\\ \{\leftarrow
              \alpha\}, & \text{otherwise.}
            \end{cases}
    \]
Accordingly, a navigation step from $\lp$ to $\lp'$ is obtained by modifying
$\lp$ such that  $\lp' = \lp \cup \mathit{ic}(f)$. Faceted navigation w.r.t. some
program $\lp$ is possible as long as $\F \neq \emptyset$.
\citex{10.1007/978-3-319-99906-7_14} established that if $f \in \F$, then $\lp'
\coloneqq \lp \cup ic(f)$ is satisfiable and $\AS[\lp'] = \{X \in \AS \mid X \models
f\}$. When referring to $\AS$ as a solution space, we refer to the topological
space induced by $2^{\AS}$ on $\AS$. Thus, answer set navigation means choosing
among subsets of answer sets.

\paragraph{Computational Complexity.}
We assume that the reader is familiar with the main concepts of
computational complexity theory~\cite{Papadimitriou94,AroraBarak09}
and follows standard terminology in the area of counting complexity
\cite{DurandHermannKolaitis05,HemaspaandraVollmer95a}.
Recall that \Ptime and \NP are the complexity classes of all
deterministically and non-deterministically polynomial-time solvable
decision problems~\cite{Cook71}, respectively.
For a complexity class~$\text{C}$, \text{co-C} denotes the class of
all decision problems whose complement is in $\text{C}$.
We are also interested in the polynomial
hierarchy~\cite{StockmeyerMeyer73,Stockmeyer76,Wrathall76} defined as
follows:
$\Delta^p_0 \coloneqq \Pi^p_0 \coloneqq \Sigma^p_0 \coloneqq \Ptime$
and $\Delta^p_i \coloneqq P^{\Sigma^p_i}$,
$\Sigma^p_i \coloneqq \NP^{\Sigma^p_i}$,
$\Pi^p_i \coloneqq \co\NP^{\Sigma^p_i}$ for $i>0$ where $C^{D}$ is the
class~$C$ of decision problems augmented by an oracle for some
complete problem in class $D$.
Further, $\PH \coloneqq \bigcup_{k \in \mathbb{N}} \Delta^p_k$.
Note that $\NP = \Sigma^p_1$, $\co\NP = \Pi^p_1$,
$\Sigma^p_2 = \NP^\NP$, and $\Pi^p_2 = \co\NP^\NP$.
If $\mathcal C$ is a decision complexity class then
$\#\cdot\mathcal C$ is the class of all counting problems whose
witness function $w$ satisfies (i) $\exists$ polynomial $p$ such that
for all $y\in w(x)$, we have that $|y|\leqslant p(|x|)$, and (ii) the
decision problem ``given $x$ and $y$, is $y\in w(x)$?'' is in
$\mathcal C$.
A \emph{witness} function is a function
$w\colon\Sigma^*\to\mathcal P^{<\omega}(\Gamma^*)$, where $\Sigma$ and
$\Gamma$ are alphabets, mapping to a finite subset of $\Gamma^*$. Such
functions associate with the counting problem ``given $x\in\Sigma^*$,
find $|w(x)|$''.

\section{Routes and Navigation Modes}
We introduce \emph{routes} as a notion for characterizing sequences of navigation steps.
\begin{definition}\label{def:routes}
    A \emph{route} $\delta$ is a finite sequence $\route[f_1, \dots, f_n]$ of
    facets $f_i \in \F$ such that $0 \leq i \leq n \in \mathbb{N}$, denoting $n$
    arbitrary navigation steps over $\lp$. We say $\delta$ is a \emph{subroute}
    of $\delta'$, denoted by $\delta \sqsubseteq \delta'$, whenever if $f_i \in
    \delta$, then $f_i \in \delta'$.  We define
    $\lp^{\delta} \coloneqq \lp \cup \mathit{ic}(f_1) \cup \dots \cup
    \mathit{ic}(f_n)$.  By $\De{}$ we denote all possible routes over $\AS$,
    including the empty route $\epsilon$.
\end{definition}
\noindent
It is easy to see that any permutation of navigation steps of a fixed set of
facets always leads to the same solutions. In general, different routes may
lead to the same subset of answer sets.
We say two routes $\delta, \delta' \in \De{}$ are equivalent if and only if
$\AS[\lp^{\delta}] = \AS[\lp^{\delta'}]$.
\noindent To ensure satisfiable programs, we aim to select so called \emph{safe}
routes.
By $\De{s} \coloneqq \{\delta \in \De{} \mid \AS[\lp^{\delta}] \neq \emptyset\}$ we
define \emph{safe routes} over $\AS$.
Once an unsafe route is taken, some sort of \emph{redirection}, which relates to
the notion of \emph{correction sets} \cite{10.1007/978-3-319-99906-7_14}, i.e.,
a route obtained by retracting conflicting facets, is required to continue
navigation.
For a program $\lp$, $\delta \in \De{}$ and $f \in \F$. 
    We denote all \emph{redirections} of $\delta$ with respect to $f$ by
    $\RE{\delta} \coloneqq \{\delta' \sqsubseteq \delta \mid f \in \delta',
    \AS[\lp^{\delta'}] \neq \emptyset\} \cup \{\epsilon\}$.
The following example illustrates faceted navigation.
\begin{example}\label{ex:Pi1}
  Consider program~$\lp_1 = \{a\,|\,b; c\,|\,d \leftarrow b; e\}$.  It is easy to
  observe that the answer sets are $\AS[\lp_1] = \{\{a, e\}$, $\{b, c, e\}$,
  $\{b, d,e\}\}$. Thus, we can choose from facets $\F[\lp_1] = \{a, b, c, d,
  \overline{a}, \overline{b}, \overline{c}, \overline{d}\}$. As illustrated in
  Figure~\ref{fig:gofree}, if we activate facet $a$ we land at
  $\AS[\lp_1^{\route[a]}] = \{\{a, e\}\}$. Activating $b$ on $\route[a]$ gives
  $\AS[\lp_1^{\route[a, b]}] = \emptyset$. To redirect $\route[a, b]$ we can
  choose from $\RE[b]{\route[a, b]}= \{\route[b]\}. $
\end{example}
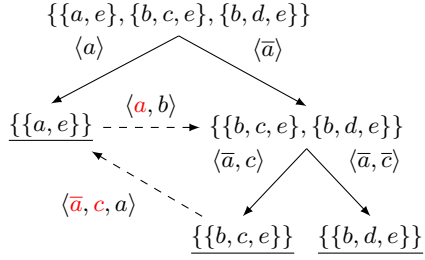
\begin{figure}
    \centering
    \begin{tikzpicture}[
                    level/.style={sibling distance=34mm/#1}, >=latex,
      ]

            \node (a) {\small $\{\{a, e\}, \{b, c, e\}, \{b, d,
              e\}\}$} child { node (b) {\underline{\small $\{\{a,
                  e\}\}$}} edge from parent [->] node [above left]
              {\small $\langle a \rangle$} } child { node (c) {\small
                $\{\{b, c, e\}, \{b, d, e\}\}$} child { node (d)
                {\underline{\small $\{\{b, c, e\}\}$}} edge from
                parent [->] node [above left] {\small $\langle
                  \overline{a}, c \rangle$} (d) edge [->, dashed] node
                [below left] {\small $\langle
                  \textcolor{red}{\overline{a}}, \textcolor{red}{c}, a
                  \rangle$} (b) (c) edge [<-, dashed] node [above]
                {\small $\langle \textcolor{red}{a}, b\rangle$} (b) }
              child { node (e) {\underline{\small $\{\{b, d, e\}\}$}}
                edge from parent [->] node [above right] {\small
                  $\langle \overline{a}, \overline{c} \rangle$} } edge
              from parent [->] node [above right] {\small $\langle
                \overline{a} \rangle$} };
    \end{tikzpicture}
    \caption{Goal-oriented and free navigation on program $\lp_1$.}
    \label{fig:gofree}
\end{figure}
\noindent We consider two more notions for identifying routes that point to a
unique solution.  A set of facets is a delimitation, if any safe route
constructible thereof leads to a unique answer set. This means that any further
step would lead to an unsafe route.
\begin{definition}\label{def:delimitation}
  Let $\lp$ be a program and $F, F' \subseteq \F$ \shortversion{such that}\longversion{such that} $F \coloneqq \{f_1, \dots,
  f_n\}$.  We define $\tau(F)$ as all permutations of $\delta \coloneqq \route[f_1,
  \dots, f_n]$ and say $F$ is \emph{delimiting} with respect to $\lp$, if
  $\tau(F) \subseteq \De{s}$ and $\forall F' \supset F: \tau(F') \not \subseteq
  \De{s}$.  By $\DF \subset 2^{\F}$ we denote the set of \emph{delimitations}
  over $\F$.
\end{definition}
\noindent
We call a route consisting of delimiting facets \emph{maximal safe}.
\begin{definition}\label{def:maxsafe}
    Let $\lp$ be a program, $F \subseteq \F$ and $\delta \in \tau(F) \subseteq
    \De{}$. We call $\delta$ \emph{maximal safe}, if and only if $F \in \DF$. By
    $\De{\mathit{ms}}$ we denote the set of maximal safe routes in $\AS$.
\end{definition}
\noindent In fact, each delimitation corresponds to a unique solution.
\shortversion{\begin{lemma}[$\star$\footnote{Statements marked by ``$\star$'' are proven in the extended version: \url{some_url}}]}
\longversion{\begin{lemma}}
	\label{lem:maxsafecar}
	Let $\lp$ be a program, $F \subseteq \F$ and $\delta \in \tau(F) \subseteq
	\De{}$. If $\delta \in \De{\mathit{ms}}$, then $|\AS[\lp^{\delta}]| = 1$.
\end{lemma}
\begin{proof}%
  Let $\lp$ be a program, $F, F' \subseteq \F$ and $\delta \in \tau(F) \subseteq
  \De{}$. Suppose $\delta \in \De{ms}$. Then $F \in \DF$ so that
  $\tau(F) \subseteq \De{s}$ and $\forall F' \supset F: \tau(F') \not \subseteq
  \De{s}$. Since $\tau(F) \subseteq \De{s}$, we have that $|\AS[\lp^\delta]| > 0$. Note that
  $\F[\lp^{\delta}] \subseteq \F$. By assumption we have $\forall F' \supset F:
  \tau(F') \not \subseteq \De{s}$, hence there is no facet $f \in \F \setminus
  F$ that can be activated in a way that $\lp^{\delta}$ would not become
  unsatisfiable, so that $\F[\lp^{\delta}] = \emptyset$. Now suppose
  $\AS[\lp^{\delta}] > 1$.  Then $|\F[\lp^{\delta}]| = |\BC \setminus \CC| > 0$,
  which contradicts $\F[\lp^{\delta}] = \emptyset$ and concludes the proof.
    
\end{proof}

\begin{theorem}
	\label{thm:delimitations} 
	$|\AS| = |\DF|$.
\end{theorem}
\begin{proof}%
  \leavevmode
  Let $\lp$ be a program and $F, F' \subseteq \F$. We need to show that $g: \DF
  \rightarrow \AS$ defined by $g(F) \coloneqq \bigcup \AS[\lp^{\delta}]$ such that  $\delta
  \in \tau(F)$ is bijective. Note that $g$ is a total function, since by
  Definition~\ref{def:maxsafe} we have $\delta \in \De{\mathit{ms}}$ and due to
  Lemma~\ref{lem:maxsafecar}, if $\delta \in \De{\mathit{ms}}$, then
  $|\AS[\lp^{\delta}]| = 1$, so that $g(F) = \bigcup \AS[\lp^{\delta}] \in \AS$.
    
  \paragraph{Injectivity:}
  Let $F, F' \in \DF$, $\delta \in \tau(F)$, $\delta' \in
  \tau(F')$ and $X, X' \subseteq \BC$. Suppose $F \neq
  F'$. It is easy to see that answer sets delimited by $F, F'$
  respectively are of the form $\bigcup \AS[\lp^{\delta}] = X
  \cup \CC$ and $\bigcup \AS[\lp^{\delta'}] = X' \cup \CC$
  such that $\forall f \in F: X \models f$ and $\forall f' \in F': X'
  \models f'$. However, since by assumption $F, F' \in \DF$ and $F
  \neq F'$, there exists a facet $f'' \in F \cup F'$ that is not
  satisfied by both $X$ and $X'$, hence $X \neq X'$, so that $\bigcup
  \AS[\lp^{\delta}] \neq \bigcup
  \AS[\lp^{\delta'}]$. Therefore by contraposition, if $g(F)
  = g(F')$, then $F = F'$.
  \paragraph{Surjectivity:} We need to show that $\forall X \in \AS \exists F \in \DF:
  g(F) = X$. Let $X \in \AS$ and $F' \subseteq \FI \subseteq \BC$ be
  an arbitrary set of inclusive facets of $\lp$. Note that, since $F'
  \subseteq \BC$, we can characterize any answer set $X \in \AS$ by $X
  = F' \cup \CC$.  We can make the distinction of cases:
  \begin{enumerate}
    \item Suppose $F' \neq \emptyset$. Then, since $F' \subseteq \FI$,
          there exists at least one route $\delta' \in \tau(F') \subseteq \De{}$ such that 
          $\bigcup \AS[\lp^{\delta'}] = X = F' \cup \CC$. It is easy
          to see that we can extend $F'$ to $F''$ by adding all facets $\overline{\alpha} \in \FE$ such that 
          $\alpha \not \in F'$, thus $X \models \overline{\alpha}$, in order to obtain a maximal 
          safe route $\delta'' \in \tau(F'') \subset \De{ms}$, which points to $X$. Therefore $g(F'') = X$.
    \item Suppose $F' = \emptyset$. Then $X = \CC$. Note that $\forall \alpha
          \in \F: \emptyset \models \overline{\alpha} \text{ and } \emptyset \not \models \alpha$. 
          Therefore routes to reach $X$ by must
          contain at least all exclusive facets $f \in \FE$ and no inclusive
          facets $f' \in \FI$ of $\lp$, hence we can conclude that if $\delta \in \tau(\FE)$, then
          $\bigcup \AS[\lp^{\delta}] = X$. It is easy to see that if
          a supersequence $\delta'$ of $\delta$ contains no inclusive facet, then
          $\delta'$ is equivalent to $\delta$, and otherwise $\delta'$ is not safe.
          Therefore $\delta$ has to be maximal safe and $\FE$ has to be 
          delimiting, hence $g(\FE) = X$. 
  \end{enumerate}
  Since $g$ is a bijection, we conclude $|\AS| = |\DF|$.
    
\end{proof}
As mentioned, using routes and facets, there are several ways to explore
solutions. A \emph{navigation mode} is a function that prunes the solution space
according to a search strategy that involves routes and facets.
\begin{definition}
  Let $X_i \in 2^{\De{}} \cup 2^{\F}$ where $0 \leq i \leq n \in \mathbb{N}$. A
  \emph{navigation mode} is a function $$\nu: X_0 \times \dots \times X_n
  \rightarrow 2^{\AS}$$ that maps an $n$-ary Cartesian product over subsets of
  routes over $\lp$ and facets of $\lp$ to answer sets of $\lp$.
\end{definition}
\noindent The idea of \emph{free} and \emph{goal-oriented} navigation was
mentioned by~\citex{10.1007/978-3-319-99906-7_14}.  While  free navigation
follows no particular strategy, during goal-oriented navigation, we narrow down
the solution space.  Next, we formalize the goal-oriented navigation mode.
\begin{definition}\label{def:go}
	We define the \emph{goal-oriented} navigation mode $\vgo:
		\De{s} \times \F \rightarrow 2^{\AS}$ by:
	\[
		\vgo[(\delta, f)] \coloneqq
		\begin{cases}
			\AS[\lp^{\route[\delta, f]}], & \text{ if } f \in \F[\lp^{\delta}];\\
			\AS[\lp^{\delta}],             & \text{otherwise.}
		\end{cases}
	\]
\end{definition}
\noindent As illustrated in Figure \ref{fig:gofree}, while during goal-oriented
navigation (indicated by solid lines) the space is being narrowed down, until
some unique solution (indicated by underscores) is found, in free mode
(indicated by both dashed and solid lines) unsafe routes are being redirected,
as illustrated on route $\route[a, b]$ where $a$ is retracted.  We call the
effect of narrowing down the space \emph{zooming in}, the inverse effect
\emph{zooming out} and any effect where the number of solutions remains the
same, \emph{slide} effect, e.g., activating $a$ on route $\route[\overline{a},
c]$.

\section{Weighted Faceted Navigation}
During faceted navigation, we can zoom in, zoom out or slide. However, we are
unaware of how big the effect of activating a facet will be.  Recall that
different routes can lead to the same unique solution.  The activation of some
facet may lead to a unique solution more quickly or less quickly than the
activation of another facet, which means that during navigation one has no
information on the length of a route. Our framework provides an approach for
consciously zooming in on solutions. Introducing \emph{weighted} navigation, we
characterize a navigation step with respect to the extent to which it affects
the size of the
solution space, thereby we can navigate toward solutions at a configurable
``pace'' of navigation, which we consider to be the extent to which the current
route zooms into the solution space.

The kind of parameter that allows for configuration is called the \emph{weight}
of a facet. Weights of facets enable users to inspect effects of facets at any
stage of navigation, which allows for navigating more interactively in a
systematic way.  Any weight or pace is associated with a \emph{weighting
function} that can be defined in various ways, specifying the number of
program-related objects, e.g., answer sets.
\begin{definition}
    Let $\lp$ be a program, $\delta \in \De{}$, $f \in \F$ and $\delta' \in
    \RE{\delta}$. We call $\#: \{\lp^{\delta} \mid \delta \in \De{}\}
    \rightarrow
    \mathbb{N}$ a \emph{weighting function}, whenever $\#(\lp^{\delta}) > 0$, if
    $|\AS[\lp]| \geq 2$. The weight $\omega_{\#}$ of $f$ with respect to $\#$,
    $\lp^{\delta}$ and $\delta'$ is defined as:
    \[
    \omega_{\#}(f, \lp^{\delta}, \delta') \coloneqq
    \begin{cases}
            \#(\lp^{\delta}) - \#(\lp^{\delta'}), &\text{ if }
            \route[\delta, f] \not \in \De[\lp^{\delta}]{s} \text{and } \delta' \neq \epsilon;\\ \#(\lp^{\delta}) -
            \#(\lp^{\route[\delta, f]}), &\text{otherwise.}
    \end{cases}
    \]
\end{definition}
\noindent The pace indicates the zoom-in effect of a route with respect to a
weighting function.
\begin{definition}\label{def:pace}
    Let $\lp$ be a program such that $|\AS| \geq 2$ and $\delta \in \De{s}$.  We
    define the pace $\pace{\#}{\delta}$ of $\delta$ with respect to $\#$ as
    $\pace{\#}{\delta} \coloneqq \nicefrac{\#(\lp) - \#(\lp^{\delta})}{\#(\lp)}$.
\end{definition}
Before we instantiate weights with actual weighting functions, we identify
desirable properties of weights.
Most importantly, weights should indicate zoom-in effects of facets on safe
routes,~i.e., a weight should identify which facets lead to a proper sub-space
of answer sets.
\begin{definition}
    We call a weight $\omega_{\#}$ \emph{safe-zooming}, whenever if $f
    \in \F[\lp^{\delta}]$, then $\omega_{\#}(f, \lp^{\delta}, \epsilon) > 0$
    for $\delta \in \De{s}$.
\end{definition}
\noindent Essentially, whenever a weight is \emph{safe-zooming} it is useful to
to inspect zoom-in effects during goal-oriented navigation. 
\begin{definition}
    We call a weight $\omega_{\#}$ \emph{splitting}, if $\#(\lp^{\delta}) =
    \omega_{\#}( \alpha, \lp^{\delta}, \delta') + \omega_{\#}(
    \overline{\alpha}, \lp^{\delta}, \delta')$ for $\delta, \delta' \in \De{s}$
    and $\alpha, \overline{\alpha} \in \F[\lp^{\delta}]$.
\end{definition}
\noindent\emph{Splitting} weights are useful during goal-oriented
navigation, as any permissible route $\delta$ in $\vgo$ is safe and if
$\#(\lp^{\delta})$ and the weight of a facet $f \in \F[\lp^{\delta}]$
for $\delta \in \De{s}$ are known, we can compute the weight of the
respective inverse facet $f' \in \F[\lp^{\delta}]$ arithmetically and
thus avoid computing $\#(\lp^{\route[\delta, f']})$.
\begin{definition}
    We call a weight $\omega_{\#}$ \emph{reliable}, whenever
    $\omega_{\#}{(f, \lp^{\delta}, \epsilon)} = \#(\lp^{\delta})$ if and only if
    $\route[\delta, f] \not \in \De{s}$ for $\delta \in \De{}$ and $f
    \in \F$.
\end{definition}
\noindent The benefit of \emph{reliable} weights, on the other hand, is that
they indicate unsafe routes. Hence, reliability can be ignored during
goal-oriented navigation, but appears to be useful during free navigation.

As we are focused on narrowing down the solution space, we want to know, whether
the associated weighting function $\#$ of a weight detects maximal or minimal,
respectively,  zoom-in effects on safe routes.
\begin{definition}
    For a program $\lp$, $\delta \in \De{}$ and $f \in F$, then:
  \begin{itemize}
    \item $f$ is \emph{maximal weighted}, denoted by
    $f \in max_{\omega_{\#}}(\lp^{\delta})$, if
    $\forall f' \in \F[\lp^{\delta}]: \omega_{\#}(f, \lp^\delta, \epsilon)
    \geq \omega_{\#}(f', \lp^\delta, \epsilon)$;
    \item $f$ is \emph{minimal weighted}, denoted by
    $f \in min_{\omega_{\#}}(\lp^{\delta})$, if
    $\forall f' \in \F[\lp^{\delta}]: \omega_{\#}(f, \lp^\delta, \epsilon)
    \leq \omega_{\#}(f', \lp^\delta, \epsilon)$.
  \end{itemize}
\end{definition}
\noindent
A weight is min-inline, if every minimal weighted facet leads to a maximal
sub-space of solutions.  Analogously, a weight is max-inline, if every maximal
weighted facet leads to a minimal sub-space.
\begin{definition}
  Let $\lp$ be a program, $\delta \in \De{s}$ and $f \in
  \F[\lp^{\delta}]$.  We call a weight $\omega_{\#}$
    \begin{itemize}
      \item \emph{min-inline}, whenever $f \in
        min_{\omega_{\#}}(\lp^{\delta})$ if and only if $$\forall f' \in
        \F[\lp^{\delta}] \setminus min_{\omega_{\#}}(\lp^{\delta}):
        |\AS[\lp^{\route[\delta, f]}]| >
        |\AS[\lp^{\route[\delta,f']}]|\text{;}$$ 
      \item \emph{max-inline}, whenever $f \in
        max_{\omega_{\#}}(\lp^{\delta})$ if and only if $$\forall f' \in
        \F[\lp^{\delta}] \setminus max_{\omega_{\#}}(\lp^{\delta}):
        |\AS[\lp^{\route[\delta, f]}]| <
        |\AS[\lp^{\route[\delta,f']}]|.$$
    \end{itemize}
\end{definition}

Below, we introduce the \emph{absolute} weight of a facet, which counts answer
sets, and two so called \emph{relative} weights, which seek for  approximating
the number of solutions to compare sub-spaces with respect to their actual size,
while avoiding counting.
\subsection{Absolute Weight}\label{sec:aw}
The most natural weighting function to identify the effect of a navigation step
is to observe the number of answer sets on a route.
The absolute weight of a facet $f$ is defined as the number of solutions by
which the solution space grows or shrinks due to the activation of $f$.
  \begin{definition}
	\label{def:aw}
    The \emph{absolute weight} $\absw$ is defined by $\abs: \lp^{\delta} \mapsto |\AS[\lp^{\delta}]|$.
  \end{definition}
\begin{example}
Let us inspect Figure~\ref{fig:gofree} and the program $\Pi_1$ from Example~\ref{ex:Pi1}. As stated by $\absw(a,
\lp_1^{\route[\overline{a}, c]}, \route[a]) = 0$, activating $a$  on
$\route[\overline{a}, c]$ induces a slide. $\absw(a, \lp_1^{\route[a]}, \route[b])
= -1$. This tells us that navigating towards $b$ on $\route[a]$ zooms out by one
solution. In contrast, $\absw(b, \lp_1^{\route[c]}, \route[\overline{a}]) = 1$ means
that we zoom in by one solution.
\end{example}
\noindent By definition, the absolute weight directly reflects the effect of a navigation
step and %
satisfies all introduced properties.
\shortversion{\begin{theorem}[$\star$]}
\longversion{\begin{theorem}}
  \label{thm:awsplrel}
  The absolute weight~$\absw$ is safe-zooming, splitting, reliable, min-inline, and max-inline.
\end{theorem}
\begin{proof}%
  Let $\lp$ be a program.
  \paragraph{safe-zooming:} Follows per definition of facets and the fact that if $f \in \F$, then $\AS[\lp^{\route[f]}] = \{X \in \AS \mid X \models f\}$.
  \paragraph{reliable:} Let $\delta \in \De{}$ and $f \in \F[\lp^{\delta}]$. By
  Definition~\ref{def:aw}:
  \begin{align}
    \absw(f, \lp^{\delta}, \epsilon) = |\AS[\lp^{\delta}]| -
    |\AS[\lp^{\route[\delta, f]}]|
	\label{al:splrel0} 
  \end{align}
  \paragraph{($\Rightarrow$)} Suppose $\absw(f, \lp^{\delta}, \epsilon) =
  |\AS[\lp^{\delta}]|$.  Using~(\ref{al:splrel0}) it follows that
  $|\AS[\lp^{\route[\delta, f]}]| = 0$, therefore $\route[\delta, f]
  \not \in \De{s}$.
  \paragraph{($\Leftarrow$)} Suppose $\route[\delta, f] \not \in \De{s}$.  By
  assumption $\AS[\lp^{\route[\delta, f]}] = \emptyset$, so that
  $|\AS[\lp^{\route[\delta, f]}]| = 0$.  Therefore due
  to~(\ref{al:splrel0}), we conclude that $\absw(f, \lp^{\delta},
  \epsilon) = |\AS[\lp^{\delta}]|$.
  \paragraph{splitting:} Let $\delta, \delta' \in \De{s}$ and $\alpha,
  \overline{\alpha} \in \F[\lp^{\delta}]$.  Then, since if $f \in
  \{\alpha, \overline{\alpha}\} \subseteq \F[\lp^{\delta}]$, then
  $\AS[\lp^{\route[f]}] \neq \emptyset$, it follows that
  $\lp^{\route[\delta, \alpha]}$ and $\lp^{\route[\delta,
      \overline{\alpha}]}$ are satisfiable, which means that
  $\route[\delta, \alpha], \route[\delta, \overline{\alpha}]\in
  \De[\lp^{\delta}]{s}$. Thus Definition~\ref{def:aw}
  gives~(\ref{al:splrel0}) for $f \in \{\alpha, \overline{\alpha}\}$,
  respectively, so that $\delta'$ can be ignored. Define
  $\SA[\lp^{\delta}]{\alpha} = \{X \in \AS[\lp^{\delta}] \mid X
  \models \alpha\}$ and $\SA[\lp^{\delta}]{\overline{\alpha}} = \{X
  \in \AS[\lp^{\delta}] \mid X \models \overline{\alpha}\}$. We know
  that $\SA[\lp^{\delta}]{\alpha} = \AS[\lp^{\route[\delta, \alpha]}]$
  and $\SA[\lp^{\delta}]{\overline{\alpha}} = \AS[\lp^{\route[\delta,
        \overline{\alpha}]}]$. It is easy to see that 
  \begin{align}
     \SA[\lp^{\delta}]{\alpha} \text{ and }
    \SA[\lp^{\delta}]{\overline{\alpha}} \text{ form a partition of }
    \AS[\lp^{\delta}] \label{al:splrel1}
  \end{align}
  hence:
  \begin{align*}
    |\AS[\lp^{\delta}]| & = |\SA[\lp^{\delta}]{\alpha}| +
    |\SA[\lp^{\delta}]{\overline{\alpha}}| \\ & =
    |\AS[\lp^{\route[\delta, \alpha]}]| + |\AS[\lp^{\route[\delta,
          \overline{\alpha}]}]| \\ & = (|\AS[\lp^{\delta}]|-
    |\AS[\lp^{\route[\delta, \overline{\alpha}]}]|) +
    (|\AS[\lp^{\delta}]|- |\AS[\lp^{\route[\delta, \alpha]}]|) & &
    \text{(\ref{al:splrel1})} \\ 
    & = \absw(\overline{\alpha}, \lp^{\delta}, \delta') +
    \absw(\alpha, \lp^{\delta}, \delta') \\
    & = \absw(\alpha, \lp^{\delta}, \delta') +
    \absw(\overline{\alpha}, \lp^{\delta}, \delta')
  \end{align*}
 \paragraph{min-inline:} Follows directly from Definition~\ref{def:aw}.
 \paragraph{max-inline:} Follows directly from Definition~\ref{def:aw}.

\end{proof}
\noindent Unfortunately, computing absolute weights is expensive.

\shortversion{\begin{lemma}[$\star$]\label{lem:compl:absw}}
\longversion{\begin{lemma}\label{lem:compl:absw}}
  Outputting the absolute weight~$\absw$ for a given program~$\Pi$ and
  route~$\delta$ is $\#\cdot\co\NP$-complete.
\end{lemma}
\begin{proof}%
  Membership and hardness can be easily established by the complexity
  of counting the number of answer sets of a disjunctive program
  $\lp$, which is known to be
  \#$\cdot$coNP-complete~\cite{fichte2017answer}.

\end{proof}

\subsection{Relative Weights}
Since computing absolute weights is computationally expensive (Lemma~\ref{lem:compl:absw}), we aim for less expensive
methods that still retain the ability to compare sub-spaces with respect to their size.  
Therefore, we investigate two \emph{relative
  weights}.

\paragraph{Facet Counting.} One approach to manipulating the number of
solutions and to keeping track of how the number changes over the course of
navigation, is to count facets.
\begin{definition}\label{def:rw}
	The \emph{facet-counting weight} $\fcw$ is defined by $\fc: \lp^{\delta} 
	\mapsto |\F[\lp^{\delta}]|$.

\end{definition}
Next, we establish a positive result in terms of
complexity. \longversion{Therefore, recall that}\shortversion{Recall} 
$\Delta^p_3 \subseteq \PH \subseteq
\Ptime^\numberP$~\cite{Stockmeyer76,Toda91}.

\newcommand{\at}[1]{\ensuremath{\mathcal{A}({#1})}^+}
\shortversion{\begin{lemma}[$\star$]\label{lem:compl:fcw}}
\longversion{\begin{lemma}\label{lem:compl:fcw}}
  Outputting the facet-counting weight~$\fcw$ for a given
  program~$\lp$ and route~$\delta$ is in $\Delta^p_3$. %
\end{lemma}
\begin{proof}%
  In fact, we obtain the membership result by the following
  construction.
  We have
  $|\F[\lp^{\delta}]| = |\mathcal{BC}(\lp^\delta) \setminus
  \mathcal{CC}(\lp^\delta)| = |\mathcal{BC}(\lp^\delta)| -
  |\mathcal{CC}(\lp^\delta)|$. The value of $|\mathcal{BC}(\lp^\delta)|$ is at
  most $|\at{\lp^\delta}|$ and we can compute $\mathcal{BC}(\lp^\delta)$ by
  checking for every atom~$\alpha \in \at{\lp^\delta}$ whether $\alpha$ is a brave
  consequence of $\lp^\delta$, which 
  is $\Sigma^P_2$-complete~\cite{EiterGottlob95}.
  Similar, we can check for $|\mathcal{CC}(\Pi^d)|$ whether
  $a \in \at{\Pi^d}$ is a cautious consequence of~$\Pi^\delta$, which is
  $\Pi^P_2$-complete~\cite{EiterGottlob95}. Computing the difference
  of the two integers takes time $\Theta(\log n)$.

\end{proof}
\noindent Hence, assuming standard theoretical assumptions,
counting facets is easier than counting solutions. 
However, below we show that counting facets has deficiencies, when it comes to
comprehending the solution space regarding its size.
\shortversion{\begin{lemma}[$\star$]\label{lem:ifff0}}
\longversion{\begin{lemma}\label{lem:ifff0}}
	$|\AS| \leq 1$ if and only if $|\F| = 0$.
\end{lemma}
\begin{proof}%
  Let $\lp$ be a program.  
  \paragraph{($\Rightarrow$)} Suppose $|\F| > 0$. Then
  $\BC \neq \emptyset$, so that $|\AS| > 0$. Now, suppose $|\AS| = 1$.  Then
  $\BC = \CC$, which means that $|\F| = 0$ and contradicts $|\F| > 0$. 
  Therefore $|\AS| > 1$, which by contraposition concludes the proposition.
  \paragraph{($\Leftarrow$)} Suppose $|\F| = 0$. Then $\BC = \CC$.  Due to the
  minimality of answer sets we conclude that therefore either $\AS = \emptyset$,
  so that $|\AS| = 0$, or $|\AS| = 1$.  Therefore $|\AS| \leq 1$.
    
\end{proof}

\noindent From Lemma~\ref{lem:ifff0} and the fact that for program $\lp_1$  from Example~\ref{ex:Pi1} we have
$\fcw(c, \lp_{1}^{\route[\overline{a}]}, \epsilon) =
|\F[\lp_{1}^{\route[\overline{a}]}]|$, but
$\route[\overline{a}, c] \in \De[\lp_1]{s}$, we conclude that $\fcw$
is not reliable.
Furthermore, since therefore
$\fcw(c, \lp_{1}^{\route[\overline{a}]}, \epsilon) +
\fcw(\overline{c}, \lp_{1}^{\route[\overline{a}]}, \epsilon) \neq
|\F[\lp_{1}^{\route[\overline{a}]}]|$, $\fcw$ is not splitting either.

\begin{corollary}
  The facet-counting weight~$\fcw$ is not reliable and not splitting.
\end{corollary}

 The reason for $\fcw$ not distinguishing between one and no
solution is that we can interpret it as an indicator for how the
diversity or similarity, respectively, of solutions changes by activating a
facet. Accordingly, whenever a step leads to one or no solution, the
thereby reached sub-space contains least-diverse or most-similar solutions, respectively.
\begin{example}\label{exm:rw1}
Again consider $\Pi_1$ from Example~\ref{ex:Pi1}.
    While on the absolute level $\absw(\overline{a},
    \lp_1, \epsilon) = 1 = \absw(\overline{c}, \lp_1, \epsilon)$,
    counting facets, $\fcw(\overline{a}, \lp_1, \epsilon) = 4$ and
    $\fcw(\overline{c}, \lp_1, \epsilon) = 2$, the relative
    weights of $\overline{c}$ and $\overline{a}$ differ. The
    reason is that even though $|\AS[\lp_{1}^{
        \route[\overline{a}]}]| = |\AS[\lp_{1}^{
        \route[\overline{c}]}]|$, by activating $\overline{c}$ we
    can still navigate towards $\F[\lp_{1}^{\route[\overline{c}]}]
    = \{a, \overline{a}, b, \overline{b}, d, \overline{d}\}$, but
    activating $\overline{a}$, we can only navigate toward
    $\F[\lp_{1}^{\route[\overline{a}]}] = \{c, \overline{c}, d,
    \overline{d}\}$, i.e., answer sets that contain $b$.
\end{example}
\noindent In other words, while $\fc$ indicates how ``far apart''
solutions are, $\fcw$ indicates to what amount the solutions converge
due to navigation steps.

\shortversion{\begin{theorem}[$\star$]\label{thm:fcwzoomin}}
\longversion{\begin{theorem}\label{thm:fcwzoomin}}
    \longversion{The facet-counting weight~}\shortversion{Weight~}$\fcw$ is safe-zooming.
\end{theorem}
\begin{proof}%
  Let $\lp$ be a program and $\delta \in \De{s}$. 
  By Definition~\ref{def:rw}:
  \begin{align}
    \fcw(f, \lp^{\delta}, \epsilon) & = |\F[\lp^{\delta}]| - |\F[\lp^{\route[\delta, f]}]| \label{al:1} 
  \end{align}
  Suppose $f \in \{\alpha, \overline{\alpha}\} \subseteq \F[\lp^{\delta}]$.
  Then we know that $\AS[\lp^{\route[\delta, f]}] =
  \{X \in \AS[\lp^{\delta}] \mid X
  \models f\}$, so that either $\forall X \in \AS[\lp^{\route[\delta,
  f]}]: \alpha \in X$, or $\forall X \in \AS[\lp^{\route[\delta, f]}]: \alpha
  \not \in X$.  Therefore either $\alpha \in \bigcap \AS[\lp^{\route[\delta,
  f]}] = \CC[\lp^{\route[\delta, f]}]$, or $\alpha \not \in \bigcup
  \AS[\lp^{\route[\delta, f]}] = \BC[\lp^{\route[\delta, f]}]$. Per definition of facets
  in both cases $\F[\lp^{\route[\delta, f]}]
  \subseteq \F[\lp^{\delta}] \setminus \{f\}$.  Therefore
  $|\F[\lp^{\route[\delta, f]}]| < |\F[\lp^{\delta}]|$.  Using (\ref{al:1})
  gives $\fcw(f, \lp^{\delta}, \epsilon) > 0$, which concludes the proof.
    
\end{proof}
\noindent Due to Theorem~\ref{thm:fcwzoomin}, we know that $\fc$ can be
used to determine the pace of safe navigation.
 In fact the facet-counting pace $\mathcal{P}_{\fc}$ emphasizes
    that $\fcw$ is not directly related to the size of the solution
    space. 
\begin{example}\label{exm:rpace}
Consider $\Pi_1$ from Example~\ref{ex:Pi1}.
   While $|\AS[\lp_1^{\route[\overline{c}]}]| = 2$ and
    $|\AS[\lp_1]| = 3$, which means that activating $\overline{c}$ on
    $\lp_1$ we lose 1 of 3 solutions so that
    $\pace{\abs}{\route[\overline{c}]} = \nicefrac{1}{3}$, we have
    $\pace{\fc}{\route[\overline{c}]} = \nicefrac{1}{4}$.
\end{example}

\noindent From Lemma~\ref{lem:ifff0}, we immediately conclude:
\begin{corollary}\label{cor:fcpace}
  $\mathcal{P}_{\fcw}(\delta) = 1$ if and only if
  $\delta \in \De{\mathit{ms}}$.  In contrast, for
  all~$\delta \in \De{s}$ we have 
  \shortversion{%
    $\pace{\abs}{\delta} \leq \frac{|\AS| - 1}{|\AS|}$.
  }
  \longversion{%
    \[\pace{\abs}{\delta} \leq \frac{|\AS| - 1}{|\AS|}.\]
  }
\end{corollary}
\noindent Corollary~\ref{cor:fcpace} states that, in contrast to
$\mathcal{P}_{\abs}$, the facet counting pace $\mathcal{P}_{\fcw}$
detects whether users sit on a unique solution. More importantly it
is the better option to find a viable implementation of the pace of
navigation for our framework. While in that sense using the relative
weight~$\fcw$ is beneficial, unfortunately it is not
\emph{min-inline}.
\begin{example}\label{exm:fcwnotmin}
    We consider $\lp_2 = \{a\,|\,b\,|\,c;\; d\,|\,e \leftarrow b;\; f \leftarrow
    c\}$ where $\AS[\lp_2] = \{\{a\}, \{b, d\}, \{b, e\}, \{c,
    f\}\}$. While $\overline{a} \in min_{\fcw}(\lp_2)$ and
    $\overline{c} \not \in min_{\fcw}(\lp_2)$, we have
    $|\AS[\lp_2^{\route[\overline{a}]}]| =
    |\AS[\lp_2^{\route[\overline{c}]}]|$. Hence, the relative weight~$\fcw$ is not
    min-inline.
\end{example}
\noindent We suspect that the property max-inline is not satisfied by the weight  $\fcw$ as we observed in our
experiments  that the activation of some facets, which had no maximal $\fcw$ weight,
lead to smaller answer set spaces than the activation of facets which had maximal $\fcw$ weight.
 An actual counterexample is still open.

\paragraph{Supported Model Counting.} Another approach to comparing sub-spaces
with respect to their size, while avoiding answer set counting, is to count supported
models. An interpretation $X$ is called \emph{supported
model}~\cite{apt1988towards,AlvianoD16} of $\lp$ if $X$ satisfies $\lp$ and for all $\alpha
\in X$ there is a rule $r \in \lp$ such that $H(r) \cap X = \{\alpha\}$, $B^+(r)
\subseteq X$ and $B^-(r) \cap X = \emptyset$. By $\mathcal{S}(\lp)$ we denote
the supported models of $\lp$. It holds that $\AS \subseteq
\mathcal{S}(\lp)$~\cite{marek1992relationship}, but the converse does not hold
in general.
We define \emph{supp weights}, by which in short we refer to supported model counting weights,
accordingly as follows.

\begin{definition}\label{def:sw}
    The supp weight $\smcw$ is defined by $\smc: \lp^{\delta} \mapsto |\mathcal{S}(\lp)|$.
\end{definition}
\noindent
The \emph{positive
dependency graph} of program $\lp$ is
$G(\lp) \coloneqq (\A, \{(\alpha_1, \alpha_0) \mid \alpha_1 \in B^+(r),
\alpha_0 \in H(r), r \in \lp\})$. $\lp$ is called \emph{tight}, if $G(\lp)$ is
acyclic. If $\lp$ is tight, then models of the completion and answer sets coincide~\cite{cois1994consistency}.
\longversion{%
  Since we have $\AS = \mathcal{S}(\lp)$ for tight programs $\lp$, we can immediately 
  obtain the following corollary.
}
\begin{corollary}
  If $\lp$ is tight, then for all $f \in \F[\lp^{\delta}]$ we have
  that
  $\absw(f, \lp^{\delta}, \delta') = \smcw(f, \lp^{\delta}, \delta')$.
\end{corollary}
\noindent Due to the fact that unsatisfiable programs may have supported
models~\cite{marek1992relationship}, $\smcw$ is not
reliable. Moreover the following example shows that $\smcw$ is neither min-inline, nor max-inline.
\begin{example}\label{ex:Pi3}
We consider $\Pi_3=\{a;\; b\leftarrow a, \simnot c;\; c\leftarrow\simnot b, \sim d;\;
d\leftarrow d\}$ with $\mathcal{S}(\Pi_3)=\{\{a,b\}, \{a,c\}, \{a,b,d\}\}$ and
$\AS[\Pi_3]=\{\{a,b\},\{a,c\}\}$. The facets of $\Pi_3$ are given by
$\mathcal{F}(\Pi_3)=\{b,\overline{b}, c, \overline{c}\}$. Then, the facets $b$ and
$\overline{c}$ both have supp weight 1 and thus are minimal weighted, and the
facets $c$ and $\overline{b}$ have supp weight 2 and thus are maximal weighted. 
As $|\AS[\Pi_3^{\route[b]}]|=|\AS[\Pi_3^{\route[c]}]|=1$ we see that both the
minimal and the maximal weighted facets with respect to supp weights have the same number
of answer sets. Hence, $\smcw$ is neither min-inline, nor max-inline.
\end{example}
\noindent
Although  $\smcw$ does not satisfy min-inline and max-inline, 
it shares some properties with $\absw$ and $\fcw$.
\shortversion{\begin{lemma}[$\star$]\label{lem:nonewsms}}
\longversion{\begin{lemma}\label{lem:nonewsms}}
	\longversion{Let $\lp$ be a program and $\delta \in \De{s}$. If}%
        \shortversion{For program~$\lp$ and $\delta \in \De{s}$, if}
        $f \in
        \F[\lp^{\delta}]$, then 
        \shortversion{%
          $\mathcal{S}(\lp^{\route[\delta, f]})
          = \{X \in \mathcal{S}(\lp^{\delta}) \mid X \models f\}\subset
          \mathcal{S}(\lp^\delta)$.
        }
        \longversion{%
          \[\mathcal{S}(\lp^{\route[\delta, f]})
          = \{X \in \mathcal{S}(\lp^{\delta}) \mid X \models f\}\subset
          \mathcal{S}(\lp^\delta).\]
        }
\end{lemma}
\begin{proof}%
  Let $\lp$ be a program and $\delta \in \De{s}$. Suppose $f \in \{\alpha,
  \overline{\alpha}\} \subseteq \F[\lp^{\delta}]$. Then, we know that
  $\AS[\lp^{\route[\delta, f]}] \neq \emptyset$, so that, using the fact that
  $\AS \subseteq \mathcal{S}(\lp)$, we conclude that
  $\mathcal{S}(\lp^{\route[\delta, f]}) \neq \emptyset$, It is well known that
  an integrity constraint $\leftarrow \alpha$ can be encoded as a self-blocking
  rule $\alpha' \leftarrow \alpha, \simnot \alpha'$ where $\alpha'$ is a new
  introduced atom, so that $ic(f)$ can be encoded as $\alpha' \leftarrow \alpha,
  \simnot \alpha'$ ($\alpha' \leftarrow \simnot \alpha, \simnot \alpha'$
  respectively).  Hence, by definition of $\mathcal{S}(\lp)$, it is easy to see that
  activating $f = \alpha$ rejects any interpretation $X \in
  \mathcal{S}(\lp^{\delta})$ that contains $\alpha$. Analogously, if $f =
  \overline{\alpha}$ any interpretation that does not contain $\alpha$ is being
  rejected. Therefore we conclude that $\mathcal{S}(\lp^{\route[\delta, f]}) =
  \{X \in \mathcal{S}(\lp^{\delta}) \mid X \models f\}\subset
  \mathcal{S}(\lp^\delta)$.    

\end{proof}
\shortversion{\begin{theorem}[$\star$]\label{thm:smcwmininline}}
\longversion{\begin{theorem}\label{thm:smcwmininline}}
  \longversion{The supp weight~}\shortversion{Weight~}$\smcw$ is safe-zooming and splitting.
\end{theorem}
\begin{proof}%
  Let $\lp$ be a program, $\delta \in \De{s}$ and $f \in
  \F[\lp^{\delta}]$.
  \paragraph{safe-zooming:} Follows directly from Lemma~\ref{lem:nonewsms}.
  \paragraph{splitting:} Suppose $f \in \{\alpha, \overline{\alpha}\}$.
  Due to Lemma~\ref{lem:nonewsms} it is easy to see that
  $\mathcal{S}(\lp^{\route[\delta, \alpha]})$ and
  $\mathcal{S}(\lp^{\route[\delta, \overline{\alpha}]})$ form a
  partition of $\mathcal{S}(\lp^{\delta})$, from which, analogously to
  the proof for the splitting property of $\absw$, it follows that
  $\smcw$ is splitting.

\end{proof}
\noindent Computing supp weights is computationally easier. 
\shortversion{\begin{lemma}[$\star$]\label{lem:compl:smcw}}
\longversion{\begin{lemma}\label{lem:compl:smcw}}
  Outputting the supp weight~$\smcw$ for a given program~$\lp$ and
  route~$\delta$ is $\numberP$-complete.
\end{lemma}
\begin{proof}%
  Since we can easily compute $\smcw$ using Clark's
  completion~\cite{clark1978negation} and propositional model
  counting~\cite{Valiant79b} and vice-versa encode a SAT instance into
  a logic program while preserving the models~\cite{Niemela99}, we
  obtain membership and hardness.

\end{proof}
\noindent However, recalling Lemma~\ref{lem:compl:fcw}, note that
counting facets is still the least expensive method.

\begin{table}[h] 
\centering
    \begin{tabular}{c|c|c|c|c|c}
                                                            & \it
        ~\texttt{saf}~ & \it ~\texttt{rel}~ & \it ~\texttt{spl}~ &
        ~\texttt{min}~ & ~\texttt{max} \\ 
      \toprule %
      $\absw$ & \ding{51} &
        \ding{51} & \ding{51} & \ding{51} & \ding{51} \\ $\fcw$ &
        \ding{51} & \ding{55} & \ding{55} & \ding{55} & ?\\ $\smcw$ & %
        \ding{51} & \ding{55} & \ding{51} & \ding{55} & \ding{55} \\
    \end{tabular}
\caption{Comparing weights regarding \texttt{saf}: is safe-zooming, \texttt{spl}: is splitting,
\texttt{rel}: is reliable, \texttt{min}: is
min-inline and \texttt{max}: is max-inline.}
\label{tab:wc}
\end{table}

In summary, 
we can characterize and compare the introduced weights as given in Table~\ref{tab:wc}.
Every weight has its advantages that should be
used to leverage performance, or characterize the solution space and
its sub-spaces. While counting solutions is the most desirable choice,
computing $\absw$ is hard. Our results show that, when 
narrowing down the space by strictly pruning the maximum/minimum number
of solutions, at least for tight programs, $\smcw$ is the best choice,
as it coincides with $\absw$ while remaining less expensive. In
general, in contrast to $\absw$, relative
weights come with different use cases regarding their interpretation.
Even though $\fcw$ has deficiencies, it satisfies the most essential
property, namely being safe-zooming, and
provides information on the
similarity/diversity of solutions w.r.t. a route. To conclude, while
facet-counting is the most promising method for distinguishing 
zoom-in effects of facets regarding computational feasibility,
counting supported models of tight programs is precise about
zoom-in effects. %

\subsection{Weighted Navigation Modes}
\label{sec:wnm}
In the following, we introduce two new navigation modes, called
\emph{strictly goal-oriented} and \emph{explore}. They can be understood as
special cases of goal-oriented navigation. 
\begin{definition}
	Let $\lp$ be a program, $\delta \in \De{s}$ and $f \in
        \F$. The \emph{strictly goal-oriented} mode $\vsgo{\#}$ and
        the \emph{explore} $\vexpl{\#}$ mode are defined by:
	\[
		\vsgo{\#}{(\delta, f)} \coloneqq
		\begin{cases} 
			\AS[\lp^{\route[\delta, f]}], & \text{ if } f
                        \in max_{\omega_{\#}}(\lp^{\delta});
                        \\ \AS[\lp^{\delta}] & \text{otherwise.}
		\end{cases}\]
    \[
      \vexpl{\#}{(\delta, f)} \coloneqq
	  \begin{cases} 
		\AS[\lp^{\route[\delta, f]}], & \text{ if } f \in
                min_{\omega_{\#}}(\lp^{\delta}); \\ \AS[\lp^{\delta}] &
                \text{otherwise.}
	\end{cases}
    \]
\end{definition}
\begin{corollary}
  $\vsgo{\#}$ and $\vexpl{\#}$ avoid unsafe routes, hence we can use
  the restriction $\omega_{\#}\restrict{X}$ of $\omega_{\#}$ where $X
  \coloneqq \{(f, \delta, \epsilon) \mid f \in \F, \delta \in \De{s}\}$.
\end{corollary}
\noindent While in strictly goal-oriented mode the objective is to
``rush'' through the solution space, navigating at the highest
possible pace in order to reach a unique solution as quick as
possible, explore mode keeps the user off one unique solution as long
as possible, aiming to provide her with as many solutions as possible
to explore while ``strolling'' between sub-spaces. As a consequence,
regardless of whether absolute or relative weights are used, during
weighted navigation some (partial) solutions may be unreachable.

\begin{example}\label{exm:unrsol}
  Consider $\lp_2$ from Example~\ref{exm:fcwnotmin} where we can choose from
  facets $\F[\lp_2] = \{ a,$ b, c, d, e, f, $\overline{a}$, $\overline{b}$,
  $\overline{c}$, $\overline{d}$, $\overline{e}$, $\overline{f}\}$ and
  $max_{\absw}(\lp_2) = \{a, c, d, e, f\} = max_{\fcw}(\lp_2)$.  Thus, any
  solution $X \in \AS[\lp_2] = \{\{a\}, \{b, d\}, \{b, e\}, \{c, f\}\}$ such that  $b
  \in X$ is unreachable in $\vsgo{\abs}$ and $\vsgo{\fc}$.  \longversion{Accordingly, since}\shortversion{Since }%
  $\absw$ is splitting, it follows that $min_{\abs}(\lp_2) = \{\overline{a},
  \overline{c}, \overline{d}, \overline{e}, \overline{f}\}$. Hence, navigating
  in $\vexpl{\abs}$, one has to sacrifice either partial solution $a$, or $c$
  and $f$ right in the beginning. Furthermore, since $min_{\fcw}(\lp_2) =
  \{\overline{a}, \overline{d}, \overline{e}\}$, right in the beginning of navigating in $\vexpl{\fc}$, 
  one has to sacrifice partial solution $a$, $d$, or $e$.
\end{example}
\section{Implementation and Evaluation}
\begin{figure*}[t]
  \centering
  \begin{subfigure}{0.32\linewidth}
    \centering
    \includegraphics[width=\linewidth]{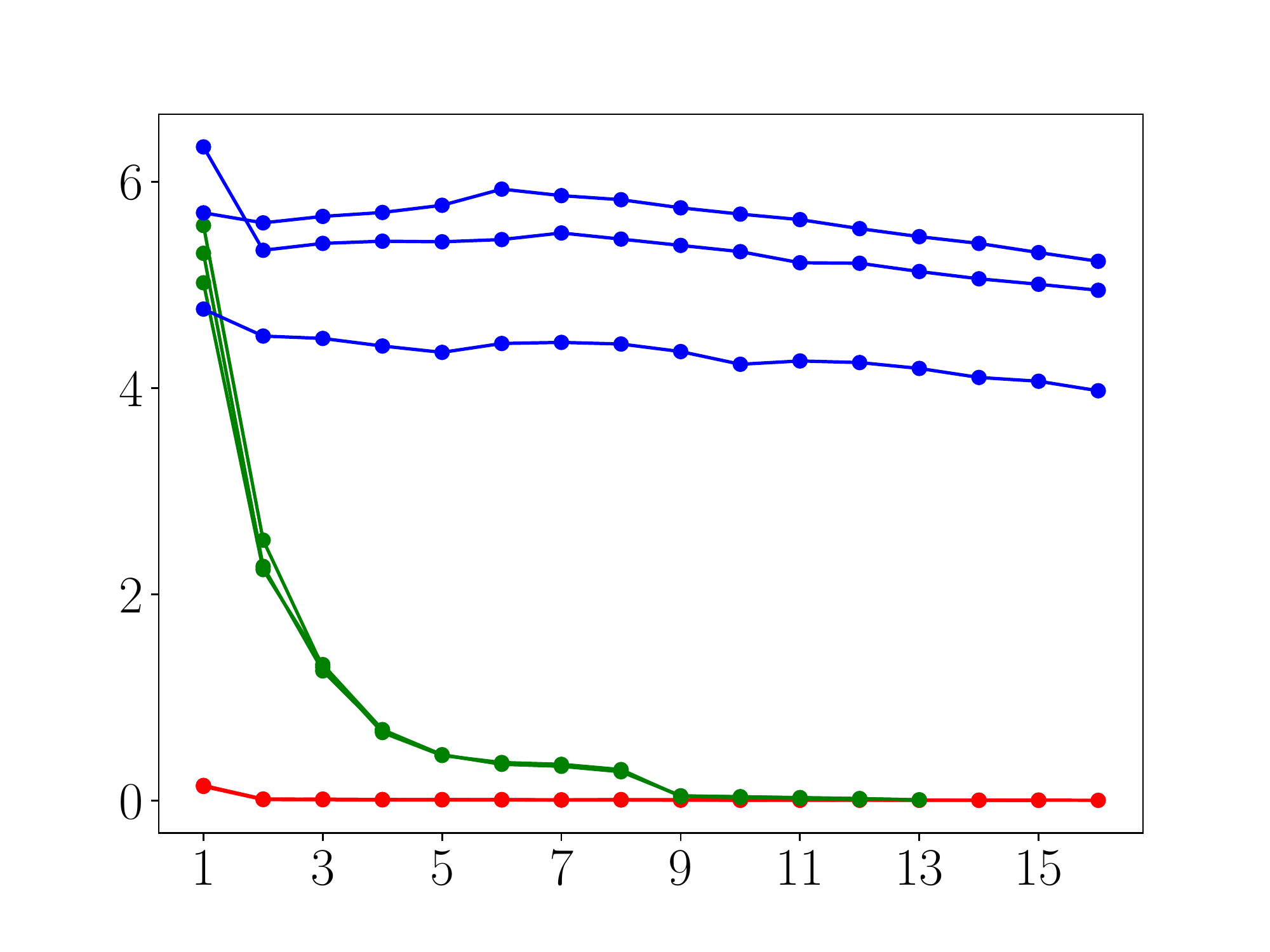}
    \caption{PC configuration.}
    \label{subfig:pcc}
  \end{subfigure}
  \begin{subfigure}{0.32\linewidth}
    \centering 
    \includegraphics[width=\linewidth]{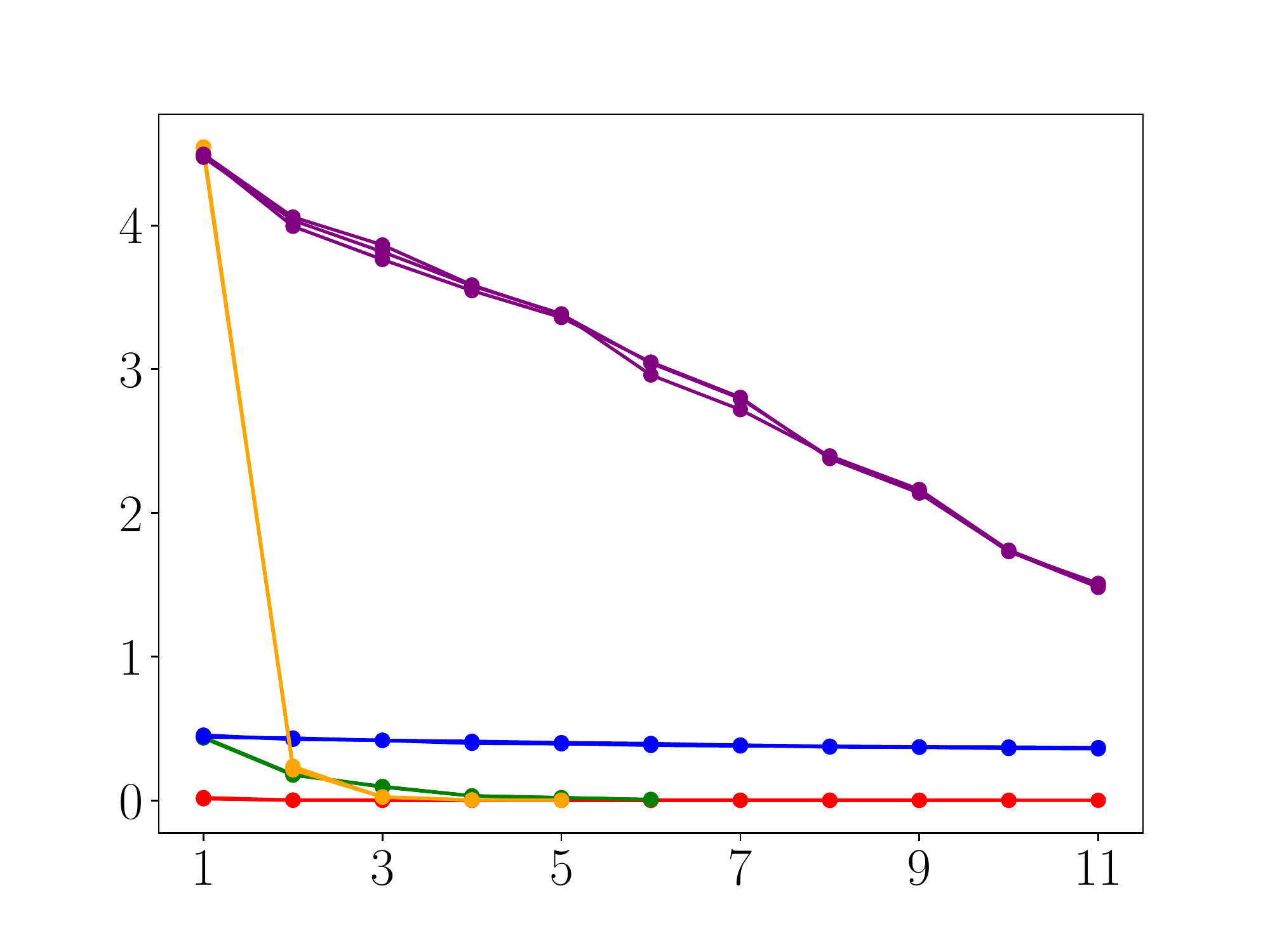}
    \caption{Stable extensions.}
    \label{subfig:afst}
  \end{subfigure}
  \begin{subfigure}{0.32\linewidth}
    \centering 
    \includegraphics[width=\linewidth]{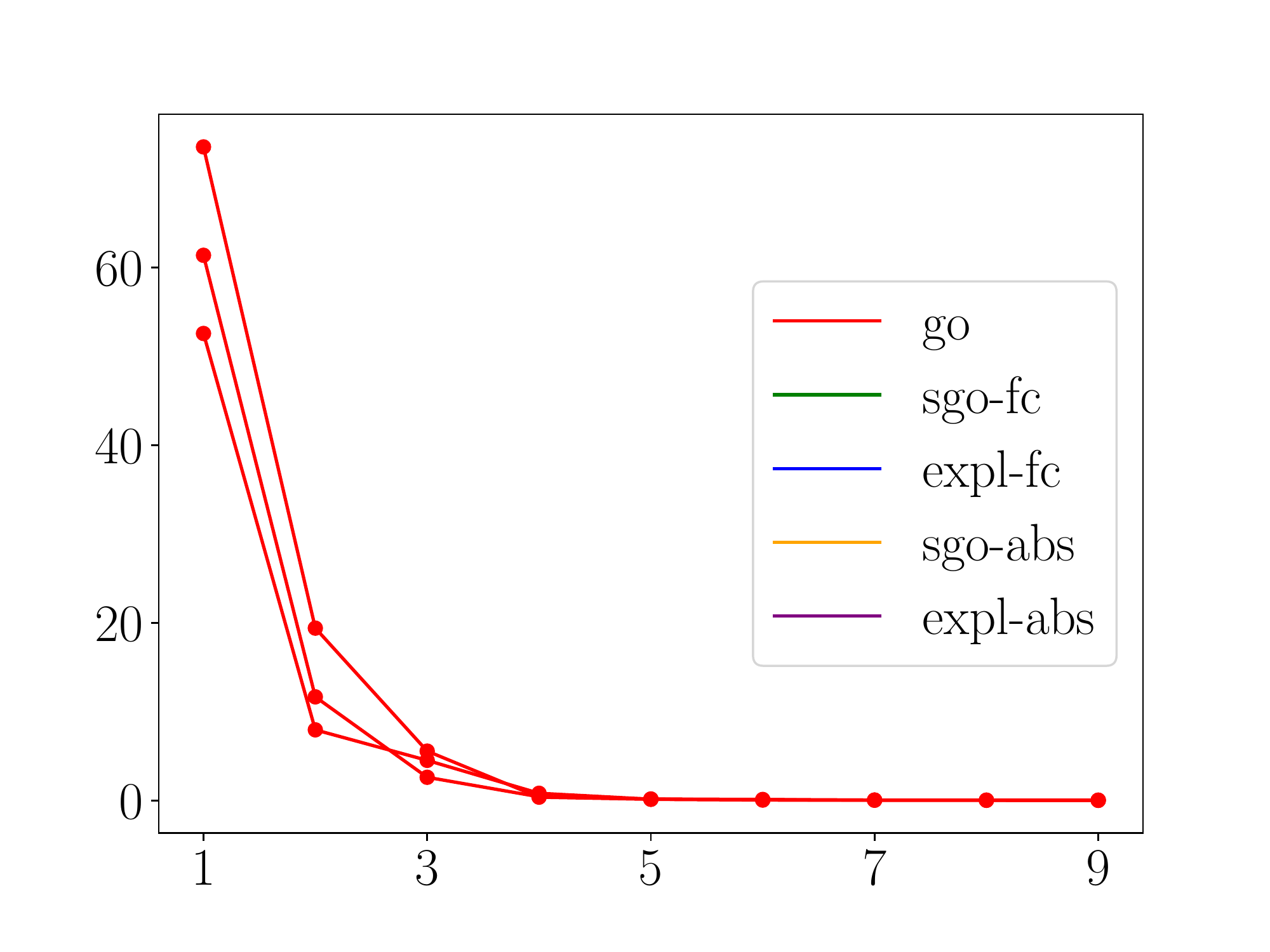}
    \caption{Preferred extensions.}
    \label{subfig:afpr}
  \end{subfigure}
  \caption{Comparing random steps in several navigation modes. %
    The x-axis refers to the respective navigation step, the y-axis refers to the execution
    time  in seconds. Colors in Figure~\ref{subfig:pcc} and~\ref{subfig:afst} follow the legend as given in
    Figure~\ref{subfig:afpr}.
  }
\end{figure*}
To study the feasibility of our framework, we implemented the \emph{faceted
answer set browser} (\verb"fasb") %
on top of the \verb"clingo" solver. In particular, we conducted
experiments on three instance sets 
that range from large solution spaces to complex encodings in order to verify the
following two hypotheses: (\textbf{H1}) weighted faceted navigation can be
performed in reasonable time in an incomprehensible solution space associated
with product configuration; and (\textbf{H2}) the feasibility of our framework
depends on the complexity of the given problem, i.e., program.
The implementation and experiments are publicly available %
\cite{2021_5780050, 2021_5767981}.

\paragraph{Environment.}
\verb"fasb" is designed for desktop systems, enabling users to
practicably explore the solution space in an interactive way. Hence,
runtime was limited to 600~seconds and the experiments were run on an
eight core Intel i7-10510U CPU 1.8~GHz with 16~GB of RAM, running
Manjaro Linux 21.1.1 (kernel 5.10.59-1-MANJARO).
Runtime was measured in elapsed time by timers in \verb"fasb" itself.

\paragraph{Design of Experiment.}
Currently, we miss data on real user behavior. Thus, we run three
iterations of random navigation steps in each of the implemented modes, to 
simulate a user and avoid bias regarding the
choice of steps.  For \emph{go}, \emph{sgo-fc}, and \emph{sgo-abs}, we use the
\verb"--random-safe-walk" call, which in the provided mode performs random steps
until the current route is maximal safe,~e.g., in \emph{sgo-fc} and
\emph{sgo-abs} it computes maximal weighted facets and then chooses one of them
to activate randomly. Since, in practice, using \emph{expl-fc} and
\emph{expl-abs}, we do not necessarily aim to arrive at a unique solution, we
use \verb"--random-safe-steps" for \emph{expl-fc} and \emph{expl-abs} and
provide the maximum number~$n$ of steps  among iterations in \emph{go}, which
performs $n$~random steps in the provided mode. We measure the elapsed time for
a mode to filter current facets according to its strategy, then, using the
mentioned calls, we randomly select a facet thereof to activate, until we
reach a unique solution or took $n$~steps.  
For any mode except \emph{go}, we ignore the elapsed time of \verb"--activate",
for \emph{go} we solely measure elapsed time of the \verb"--activate" call,
which in the case of \emph{go} includes runtime of computing facets.
\verb"fasb" computes the initial facets at startup, which are used throughout
further computations, in particular when performing a first step.  Thus, we
add elapsed time, due to startup, to the first result in each mode.

\paragraph{Instances.}
To study (\textbf{H1}), we inspect product configuration~\cite{2020_5777217}
where users may
configure PC components over a large solutions space until a full configuration
is obtained. 
To verify (\textbf{H2}), we select instances from \emph{abstract argumentation}
using the ASPARTIX fixed ASP encodings~\citep{DvorakGRWW20} \verb"stable.lp"
and \verb"preferred-cond-disj.dl". 
There, brave and cautious reasoning in abstract argumentation is of higher
complexity for preferred semantics than for the stable semantics~\cite{BaroniCG2011}. 
For the \emph{stable} argumentation semantics, the problems can be
encoded as normal programs. Whereas for preferred, one needs
disjunctive programs.
As input instance, we used the abstract argumentation framework
\verb"A/3/ferry2.pfile-L3-C1-06.pddl.1.cnf.apx" from the benchmark set of 
(ICCMA'17)~\cite{GagglLMW20}. There, solutions of both semantics coincide with 
exactly 7696 answer sets. 

\paragraph{Observations and Results.}
In the beginning of PC configuration, we choose from 340 facets resulting in on average
 in 15 steps in \emph{go} and 13 steps in \emph{sgo-fc} to reach a uniqe
solution. Taking 16 steps in \emph{expl-fc}, throughout all iterations the
facet-counting pace of the obtained route is 9\%.  The number of solutions for
the respective generated benchmark \verb"pc_config" remains unknown. Running \verb"clingo" for over 9 hours
resulted in
more than~$1.3 \cdot 10^9$~answer sets.  As expected, for more than a billion
solutions, \emph{sgo-abs} and \emph{expl-abs} timed out in the first step.
Inspecting Figure~\ref{subfig:pcc}, we see that \emph{sgo-fc} execution time
drops significantly from Step 1 to 5, 
which originates in the fact that Steps 1 to 5 throughout all iterations on the
average decreased the number of remaining facets by 35\%. Consequently, it
reduces the number of facets to compute weights for and leads to shorter
execution times. In \emph{expl-fc}, on the other hand, throughout all iterations
each step decreases the facet-count by 2. Except for one outlier, this leads to
slowly decreasing, but in general, similar execution
times. %
Figures~\ref{subfig:afst} and~\ref{subfig:afpr} illustrate the
execution times for navigation steps in  the argumentation instances.
As expected, we see no timeouts when navigating through 7696 stable
extensions.
Whereas exploring 7696 preferred extensions, works only in mode \emph{go}. 
Computing cautious consequences was most expensive when considering the
execution time of processes at startup for preferred extensions, which
emphasizes (\textbf{H2}).
From Figure~\ref{subfig:afst}, we see that \emph{go}, \emph{sgo-fc},
and \emph{expl-fc} show a similar trend to Figure~\ref{subfig:pcc}.
While \emph{go} and \emph{expl-fc} remain rather steady in execution
time, \emph{sgo-fc} drops in the first steps.  Moreover, we observe
that the execution time of \emph{expl-abs}, in contrast to
\emph{expl-fc}, decreases noticeably with every step indicating
that 
counting less answer sets in each step becomes easier, whereas counting facets does not. Throughout
all iterations, while \emph{sgo-fc} needs 6 steps, \emph{sgo-abs} only needs 5
steps to reach a unique solution.  The significant drop between Step 1 and 2 in
\emph{sgo-abs} originates in zooming
in by 93\%, pruning 7152 out of 7696 solutions.

\paragraph{Summary.} 
In general (\textbf{H2}) the feasibility of weighted
navigation depends on the complexity of the given problem.
Regarding product configuration, associated with a large and
incomprehensible solution space
(\textbf{H1}), weighted navigation can be performed in reasonable time
using \verb"fasb".

\section{Conclusion and Future Work}
We provide a formal, dynamic, and flexible framework for navigating
through subsets of answer sets in a systematic way.
We introduce absolute and relative weights to quantify the size of the
search space when reasoning under assumptions (facets) as well as
natural navigation operations. %
In a  systematic comparison, we prove  which weights can be employed
under the search space navigation operations.
In addition, we illustrate the computational complexity for computing
the weights.
Our framework is intended as an additional layer on top of
a solver, adding functionality for systematically manipulating the
size of the solution space during (faceted) answer set navigation.
Our implementation, on top of the solver \verb"clingo", demonstrates
feasibility of our framework for an incomprehensible solution space.

For future work, we believe that an interesting question is to research
relative weights which preserve the properties min-inline and max-inline.
Furthermore, we aim to investigate whether supported model counting is
in fact practically feasible using recent developments in
propositional model
counting~\cite{BM20,FichteHecherHamiti20,FichteEtAl21b,FichteHecherRoland21,KorhonenJarvisalo2021}
and ASP~\cite{FichteHecher19a}.

\cleardoublepage
\section{Acknowledgements}
The authors are stated in alphabetic order. This research was partially funded by
the DFG through the Collaborative Research Center, Grant TRR 248 see
\url{https://perspicuous-computing.science} project ID 389792660, the
Bundesministerium für Bildung und Forschung (BMBF), Grant 01IS20056\_NAVAS, a
Google Fellowship at the Simons Institute, and the Austrian Science Fund (FWF),
Grant Y698.
Work has partially been carried out while Johannes Fichte was visiting the
Simons Institute for the Theory of Computing.

\longversion{
  \bibliographystyle{named}

}
\shortversion{
}

\end{document}